\documentclass{article}

\usepackage[preprint]{neurips_2026}

\usepackage[utf8]{inputenc} 
\usepackage[T1]{fontenc}    
\usepackage{hyperref}       
\usepackage{url}            
\usepackage{booktabs}       
\usepackage{amsfonts}       
\usepackage{nicefrac}       
\usepackage{microtype}      
\usepackage{xcolor}         
\usepackage{amsthm}
\usepackage{amsmath}
\newtheorem{theorem}{Theorem}
\usepackage{graphicx}
\usepackage{multirow}

\title{DirectTryOn: One-Step Virtual Try-On via Straightened Conditional Transport}

\author{%
  Xianbing Sun \quad
  Jiahui Zhan \quad
  Liqing Zhang \quad
  Jianfu Zhang\thanks{Corresponding author}\\[2mm]
  Shanghai Jiao Tong University\\
  \texttt{\{fufengsjtu, c.sis\}@sjtu.edu.cn}
}

\begin{document}

\maketitle

\begin{figure}[!ht]
\centering
\includegraphics[width=\linewidth]{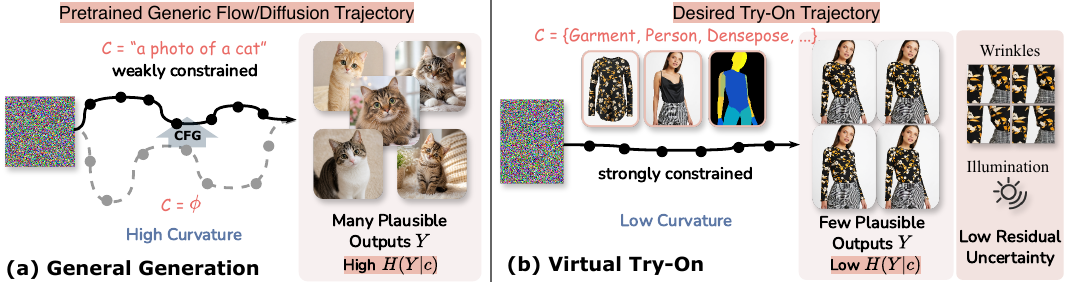}
\caption{Comparison between general image generation and virtual try-on from the perspective of conditional transport. Virtual try-on is strongly constrained by the person and garment inputs, resulting in fewer plausible outputs and lower residual uncertainty. This suggests a much straighter desired transport path, which is the core intuition of this paper.}
\label{fig:teaser}
\end{figure}

\begin{abstract}
Recent diffusion- and flow-based VTON methods achieve strong results with pretrained generative models, but their reliance on multi-step sampling incurs high inference cost, while existing acceleration methods largely overlook the intrinsic structure of the try-on task.
In this paper, we highlight a key observation: VTON outputs are highly constrained by the conditional inputs, suggesting that the conditional sampling trajectory can be much straighter than that in general image generation, making one-step generation a natural solution. However, limited task-specific data makes training from scratch impractical, forcing existing methods to fine-tune pretrained models whose objectives do not encourage such straight conditional trajectories. Thus, the deviation from an ideal straight path mainly comes from the mismatch between pretrained base models and the conditional nature of try-on generation, rather than from the task itself.
Motivated by this insight, we encourage straighter VTON sampling trajectories through three targeted modifications: pure conditional transport, a garment preservation loss, and a self consistency loss. We further introduce a one-step distillation stage. Extensive experiments show that our method achieves state-of-the-art performance with one-step sampling, establishing a new standard for efficient and high-quality VTON.
\end{abstract}

\section{Introduction}

Given a garment image and a person image, virtual try-on aims to synthesize a photorealistic image of the person wearing the specified garment~\citep{han2018viton}. As an important technology for online fashion retail and e-commerce, virtual try-on has attracted growing attention in recent years~\citep{choi2021viton,ge2021parser,gou2023taming,lee2022high,morelli2022dress,morelli2023ladi,choi2025improving,chong2025catvton,zhou2025flowattention}.

Earlier virtual try-on methods~\citep{goodfellow2020generative,choi2021viton,ge2021disentangled,lee2022high,xie2023gp} were mainly based on GANs~\citep{goodfellow2020generative}. Recent advances in diffusion models~\citep{ho2020denoising,rombach2022high} and flow matching~\citep{lipman2022flow} have shifted the field toward adapting powerful pretrained generative models, such as Stable Diffusion variants~\citep{rombach2022high,esser2024rectifiedflow} and Flux~\citep{blackforestlabs2024flux1}, for virtual try-on~\citep{kim2024stableviton,zhu2023tryondiffusion,choi2025improving,zhou2025flowattention}. Although these methods achieve superior results by leveraging strong pretrained priors, their reliance on multi-step sampling leads to substantially increased inference cost, which makes them less practical for real-world applications.

To improve inference efficiency, recent methods such as CAT-DM~\citep{zeng2024cat} and MC-VTON~\citep{luan2025mc} reduce the number of sampling steps. While effective to some extent, we argue that they do not exploit a more fundamental property of virtual try-on: the output is strongly constrained by the conditional inputs.

We first revisit this intuition from flow matching. In general, although the sample-wise optimal transport paths are straight, the marginal velocity field can become curved after marginalizing over the data distribution~\citep{lipman2022flow}. However, in an extreme conditional case where the target is exactly determined by the condition image \(y\), the conditional target distribution collapses to a single point. For each noise sample \(\epsilon\), the path is
$
x_t = t y + (1-t)\epsilon ,
$
with velocity
$
v_t(x_t\mid y)= y-\epsilon = \frac{y-x_t}{1-t}.
$
Since conditioning on \(y\) leaves no ambiguity over the endpoint, marginalization does not bend the conditional velocity field. The ideal conditional transport therefore remains straight, so one-step and multi-step sampling recover the same output \(y\).
This example suggests that when the condition largely determines the target, the conditional transport path can be much straighter than in general image generation, making one-step sampling a natural solution. \textbf{VTON is close to this regime: the try-on result is mainly determined by the person image and the garment image, while the remaining uncertainty is mostly limited to local wrinkles, shading, and lighting variations. These plausible outputs usually differ only slightly from one another.}

The practical challenge is that existing VTON models are typically adapted from pretrained generative models, such as Stable Diffusion variants~\citep{rombach2022high,esser2024rectifiedflow}, rather than trained from scratch on task-specific data. This introduces a mismatch: during pretraining, text conditions do not uniquely determine the output, and unconditional training samples are often included, leading to curved transport trajectories. As a result, directly fine-tuned models are not naturally expected to behave consistently under one-step and multi-step sampling.

Based on this view, we propose DirectTryOn, which consists of three key improvements. First, we propose pure conditional transport to better align the model with the strongly conditioned nature of virtual try-on, by removing unnecessary interference from unconditional generation during both training and inference. Second, since pretrained models tend to emphasize low-frequency information at early timesteps~\citep{park2025blockwise}, which can weaken the preservation of fine-grained garment details, we introduce a garment reconstruction objective to provide a more direct supervision signal for garment fidelity throughout the trajectory. Third, we introduce an aggressive self consistency loss to explicitly encourage the model to produce more consistent velocity predictions across timesteps, thereby promoting a straighter conditional transport path. With these improvements, the model trained with a 1000-step objective already achieves competitive one-step inference performance. Finally, we apply Latent Adversarial Diffusion Distillation (LADD)~\citep{sauer2024fast} to further bridge the gap between multi-step training and one-step generation.

Architecturally, the key challenge in pretrained diffusion- and flow-matching-based VTON is effective garment injection. Existing methods mainly follow two designs: OOTDiffusion~\citep{xu2025ootdiffusion} uses a reference U-Net~\citep{zhang2023referencecontrolnet,hu2024animate,xu2024magicanimate}, while CatVTON~\citep{chong2025catvton} concatenates the garment image with the noisy latent for self-attention-based fusion. We adopt an MMDiT style architecture~\citep{esser2024rectifiedflow}, which combines the strengths of both paradigms by enabling direct garment-person interaction while maintaining separate feature processing streams, thereby providing stronger modeling capacity for high-fidelity try-on generation.

In conclusion, our main contributions are as follows:
\begin{enumerate}
    \item 
    We formulate one-step VTON as a low-entropy conditional transport problem and analyze why its conditional trajectory can be straighter than that of general image generation.
    \item Based on this observation, we propose pure conditional transport, together with a garment preservation loss and a self consistency loss, to explicitly straighten the conditional transport trajectory and make one-step VTON generation practical.
    \item With an additional one-step distillation stage, our model achieves state-of-the-art performance on VITON-HD~\citep{choi2021viton} and DressCode~\citep{morelli2022dress} using only a single sampling step, setting a new standard for efficient and high-quality VTON.
\end{enumerate}

\section{Related works}
\label{sec:rw}

\paragraph{Virtual try-on methods.}
Earlier GAN-based~\citep{goodfellow2020generative} virtual try-on methods~\citep{choi2021viton,ge2021disentangled,lee2022high,xie2023gp} typically decompose the task into two stages: (1) warping the garment to align with the target human body shape, and (2) integrating the warped garment with the person image to synthesize the final result. In more recent diffusion-based and flow-matching-based VTON methods, a key question is how to inject garment information into the main network.
Earlier methods~\citep{morelli2023ladi,gou2023taming} typically use an additional image encoder such as CLIP~\citep{radford2021learning} and inject garment information through cross-attention. TPD~\citep{yang2024texture} emphasizes the importance of self-attention and shows that directly concatenating the conditional image with the noisy latent along the spatial dimension is more effective than using an additional image encoder with cross-attention for garment information injection. OOTDiffusion~\citep{xu2025ootdiffusion} then adopts a reference U-Net~\citep{zhang2023referencecontrolnet,hu2024animate,xu2024magicanimate} to inject garment information through self-attention layers, while FitDiT~\citep{jiang2024fitdit} can be viewed as a Transformer variant of this design. CatVTON~\citep{chong2025catvton}, by contrast, follows a strategy closer to TPD by concatenating the garment image with the noisy latent along the spatial dimension. JCo-MVTON~\citep{wang2025jco} further adopts an MMDiT style~\citep{esser2024rectifiedflow} architecture.
Beyond architectural design, existing methods can also be grouped at the training framework level into two categories: mask-based and mask-free. Mask-based methods are constrained by mask accuracy, whereas mask-free methods are constrained by the quality of training data~\citep{sun2025ds}.

\paragraph{Sampling acceleration.}
In flow matching~\citep{lipman2022flow}, the transport path is typically curved after marginalizing over the data distribution, so multi-step sampling is still needed to reach the target distribution. A natural way to reduce the number of sampling steps is therefore to straighten the transport path. This line of work is rooted in Rectified Flow~\citep{liu2023rectifiedflow}, which aims to reduce ODE integration error by making trajectories straighter. InstaFlow~\citep{liu2023instaflow} pushes this direction toward one-step and few-step generation by showing that reflow, which reassigns noise-sample pairs and further straightens the probability flow, is crucial for accelerating Stable Diffusion. PeRFlow~\citep{yan2024perflow} instead performs straightening in a piecewise manner by dividing the full trajectory into multiple time windows. OFM~\citep{kornilov2024optimal} takes a more theoretical step and aims to recover the straight optimal transport displacement under quadratic cost in a single flow matching training procedure. In the try-on setting, CAT-DM~\citep{zeng2024cat} and MC-VTON~\citep{luan2025mc} also explore sampling acceleration. CAT-DM initializes diffusion from a GAN generated coarse try-on result instead of Gaussian noise, thereby truncating the denoising trajectory, while MC-VTON adopts LADD~\citep{sauer2024fast} to distill a trained virtual try-on model into an 8-step generator.

\begin{figure}[t]
  \centering
  \includegraphics[width=\linewidth]{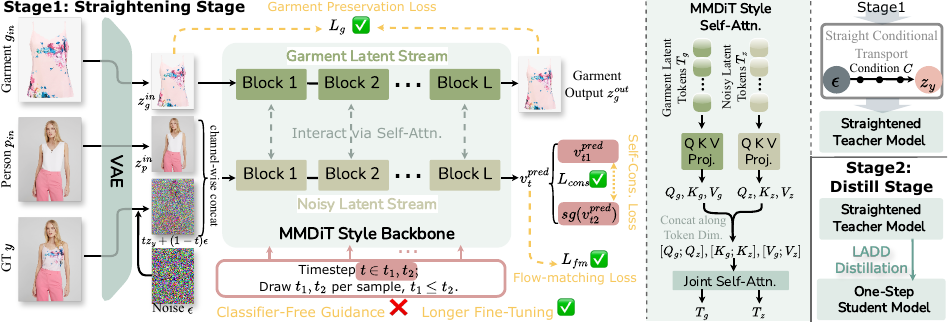}
  \caption{Overview of our framework. Stage 1 trains a teacher model to straighten the conditional transport path for virtual try-on through pure conditional transport, garment supervision, and self consistency loss. Stage 2 distills the straightened teacher into a one-step student model.}
  \label{fig:pipeline}
\end{figure}

\section{Methodology}
\label{sec:method}

\subsection{Framework overview}
Given a person image $p_{in}$ and a garment image $g_{in}$, our goal is to generate a try-on image $y$ where the person wears the specified garment. We adopt a mask-free setting, where the model directly takes the original person image as input without requiring an explicit garment mask. We perform generation in the latent space. Let $z_p^{in}$, $z_g^{in}$, and $z_y$ denote the latents of the person image, garment image, and target image, respectively.
Our framework contains two stages. In Stage 1, we train a straightened teacher model by encouraging near straight conditional transport with pure conditional transport, garment supervision, and self consistency loss. In Stage 2, we distill the straightened teacher into a one-step student model for efficient inference.
We describe the network architecture in Sec.~\ref{sec:net}, the path straightening objectives in Sec.~\ref{sec:s_stage}, and the distillation stage in Sec.~\ref{sec:distill}.

\subsection{Network architecture}
\label{sec:net}

As shown in Fig.~\ref{fig:pipeline}, we adopt an MMDiT style architecture and initialize the model weights from SD3~\citep{esser2024rectifiedflow}. We first compare our design with two representative architectures for garment information injection, namely OOTDiffusion~\citep{xu2025ootdiffusion} and CatVTON~\citep{chong2025catvton}. 

In OOTDiffusion, the reference U-Net behaves like a large garment encoder: the garment latent is extracted with its own model weights and, in principle, does not need to serve as a query in the subsequent self-attention computation or be updated during denoising. As a result, it only needs to be computed once and can then be reused throughout the entire denoising process. AnyFit~\citep{li2024anyfit} adopts this design. In CatVTON, by contrast, the garment image is concatenated with the noisy latent, so the garment latent can interact with the noisy latent through self-attention and is updated jointly with it, but it does not have its own separate weights. Consequently, the garment related computation has to be repeated at every denoising step. From this perspective, the reference U-Net design is more efficient at inference, whereas the concatenation design is simpler and has a smaller model size.

MMDiT combines the strengths of both designs: different branches can interact through self-attention, while each branch still maintains its own weights for feature processing. Although this design is computationally more expensive, our goal is one-step generation, so we prioritize model capacity and representation ability. Compared with JCo-MVTON~\citep{wang2025jco}, which also follows an MMDiT style design, we do not introduce a separate branch for the person image. Instead, we concatenate the person image with the noisy latent along the channel dimension, since preserving person identity is relatively simple and does not require an additional stream.

\begin{figure}[t]
  \centering
  \includegraphics[width=\linewidth]{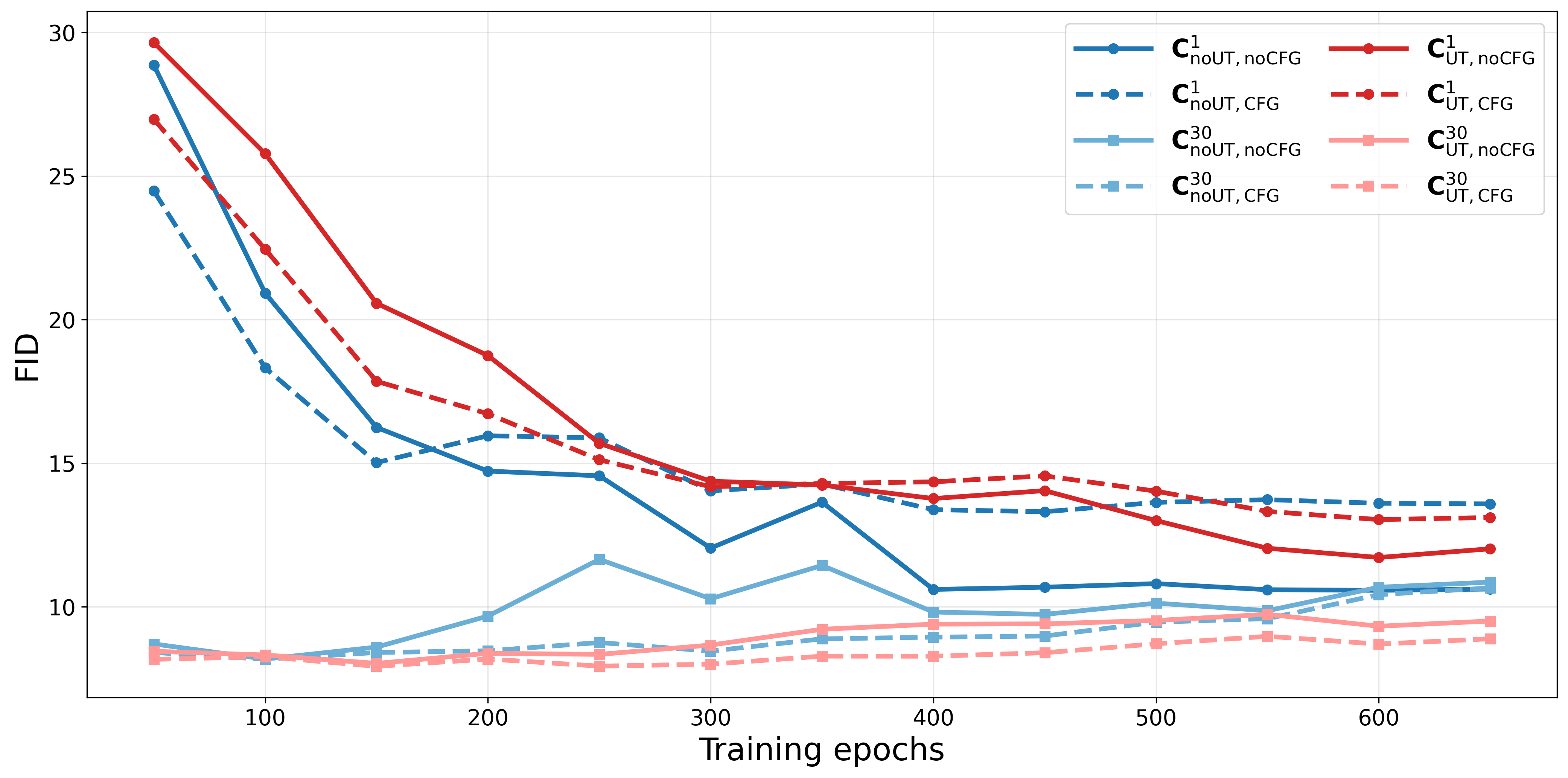}
  \caption{FID curves on VITON-HD under the unpaired evaluation setting. All models are trained in the mask-free setting, and the eight configurations are obtained by varying three factors: whether unconditional training samples are included, whether classifier-free guidance (CFG) is used during inference, and whether inference is performed with 1 or 30 sampling steps. Models are trained for 650 epochs and evaluated every 50 epochs. Detailed discussion is provided in Sec.~\ref{sec:s_stage}.}
  \label{fig:cfg_ablation}
\end{figure}

\subsection{Stage1: conditional transport path straightening}
\label{sec:s_stage}

We begin with a simple theorem that formalizes the ideal conditional case motivating our method. The proof is given in Appendix~\ref{appendix:proof_theorem1}.

\begin{theorem}
Under the Optimal Transport conditional path in flow matching, if the conditional input uniquely determines the target, then the conditional transport path is straight, and one-step sampling is equivalent to multi-step sampling.
\end{theorem}

Based on this motivation, our view is different from prior works that aim to straighten the marginal transport path itself~\citep{liu2023rectifiedflow,liu2023instaflow,yang2024consistency,kornilov2024optimal}. In VTON, the conditional structure already makes the desired transport path close to straight. Therefore, our goal is not to force an inherently curved marginal path to become straight, but to help the model better exploit the strongly constrained conditional structure of try-on generation.

\paragraph{Pure conditional transport.}
To better understand how classifier-free guidance~\citep{ho2022classifier} affects final performance, we conduct an experiment on VITON-HD~\citep{choi2021viton} based on the architecture described above. Training is performed in the mask-free setting, and evaluation is conducted under the unpaired setting using FID~\citep{parmar2022aliased}. We consider eight configurations obtained by varying three factors: whether unconditional training samples are included, whether classifier-free guidance (CFG) is used during inference, and whether inference is performed with one step or 30 steps. We denote a configuration by $\mathbf{C}_{\text{A,B}}^{K}$, where $\text{A}\in\{\text{UT},\text{noUT}\}$ indicates whether unconditional training samples are included, $\text{B}\in\{\text{CFG},\text{noCFG}\}$ indicates whether CFG is used during inference, and $K\in\{1,30\}$ denotes the number of sampling steps used at inference. In the unconditional training setting, we remove the garment condition to form the unconditional branch and the unconditional ratio is 0.2. When CFG is used during inference, the guidance scale is set to 2. We train the model for 650 epochs in total and evaluate it every 50 epochs, and the results are shown in Fig.~\ref{fig:cfg_ablation}. First, comparing $\mathbf{C}_{\text{noUT,noCFG}}^{30}$ with $\mathbf{C}_{\text{noUT,noCFG}}^{1}$, we observe that the best performance under 30-step sampling is reached at around 100 epochs, whereas the performance under one-step sampling continues to improve until around 400 epochs. This suggests that, after 100 epochs, the main improvement is no longer better modeling of the target distribution itself, but a progressively straighter conditional transport path. \textbf{This provides evidence that the try-on task naturally admits a much straighter transport path.} Although a fine-tuned pretrained model does not initially possess this property, the mismatch is gradually reduced as fine-tuning continues. Second, comparing $\mathbf{C}_{\text{UT,CFG}}^{1}$ with $\mathbf{C}_{\text{UT,noCFG}}^{1}$, and $\mathbf{C}_{\text{noUT,CFG}}^{1}$ with $\mathbf{C}_{\text{noUT,noCFG}}^{1}$, we find that \textbf{removing CFG during inference consistently improves one-step performance}, regardless of whether unconditional training samples are included. This is consistent with our analysis: the unconditional transport path tends to be more curved, and therefore under one-step sampling it is more likely to deviate from the target distribution, which harms final performance. Third, comparing $\mathbf{C}_{\text{UT,noCFG}}^{1}$ with $\mathbf{C}_{\text{noUT,noCFG}}^{1}$ shows that even if CFG is not used during inference, including unconditional training samples still leads to worse performance. This indicates that \textbf{unconditional training also makes the learned conditional transport path more curved.} Intuitively, although unconditional and conditional training are intended to model two different transport paths, they are optimized within the same model and therefore inevitably interfere with each other. Based on these observations, we remove unconditional training samples during training and do not use CFG during inference. More importantly, these experiments not only motivate the design of pure conditional transport, but also \textbf{support our central claim that the conditional transport path in try-on is naturally much straighter than that in general image generation.}

\begin{figure}[t]
  \centering
  \includegraphics[width=\linewidth]{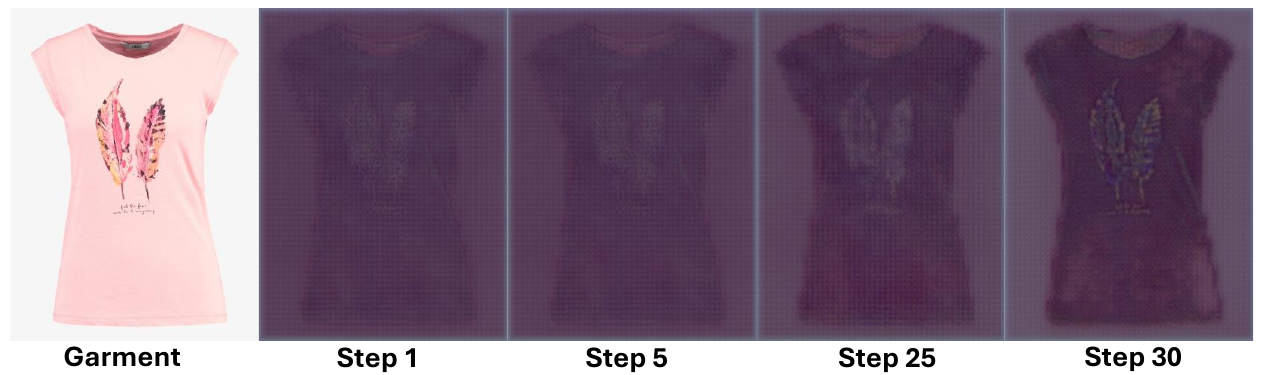}
  \caption{Visualization of garment reconstruction during 30-step inference. The leftmost image is the input garment, and the others show the reconstructed garment output from the garment branch at the 1st, 5th, 25th, and 30th sampling steps. Early steps mainly preserve coarse low-frequency structure, while richer high-frequency details emerge at later steps.}
  \label{fig:garment_freq}
\end{figure}

\paragraph{Garment preservation loss.}
In flow matching, the marginal velocity at early timesteps tends to emphasize low-frequency information~\citep{park2025blockwise}. To examine how this property affects the try-on setting, we conduct an experiment using the same architecture described in Sec.~\ref{sec:net}, without any additional modifications. The model is trained on VITON-HD for 100 epochs and evaluated with 30-step inference. Since our architecture adopts an MMDiT style design with a separate garment stream, the garment branch also produces a garment latent representation. We denote the garment latent predicted by this branch as $z_g^{out}$, and the VAE-encoded latent of the input garment as $z_g^{in}$. For visualization in Fig.~\ref{fig:garment_freq}, we decode $z_g^{out}$ back to the image space.
As shown in Fig.~\ref{fig:garment_freq}, at early timesteps the decoded garment output mainly preserves coarse low-frequency structure, while richer high-frequency details emerge at later timesteps. This suggests that pretrained model weights tend to emphasize low-frequency information at early timesteps. When adapted to virtual try-on, the model may therefore underuse or gradually lose fine-grained garment details along the trajectory. Although this issue can be alleviated to some extent with longer fine-tuning, we introduce a more direct training signal by adding a garment preservation loss in the latent space:
\[
L_g = \left\| z_g^{out} - z_g^{in} \right\|_2^2 .
\]
This objective encourages the garment branch to preserve the input garment representation at every timestep, thereby helping the model maintain garment details across the entire sampling trajectory.

\paragraph{Self consistency loss.}
To encourage the model to learn that the overall conditional transport path should be close to straight, we introduce an aggressive self consistency loss. Given two arbitrary timesteps $t_1$ and $t_2$ with $t_1 \leq t_2$, and recalling that $x_t = t y + (1 - t)\epsilon$, let $v_t^{pred}$ denote the velocity vector predicted by the model at timestep $t$. We define
\[
L_{\mathrm{cons}} =
\left\| v^{\mathrm{pred}}_{t_1}
-
\mathrm{stopgrad}\!\left(v^{\mathrm{pred}}_{t_2}\right)
\right\|_2^2 .
\]

This loss has two key characteristics. First, we do not impose any constraint on the choice of timesteps. In many previous methods that also aim to enforce consistency, such as Consistency Models~\citep{song2023consistency}, the two timesteps are usually constrained, for example by requiring a fixed interval between them. In our setting, however, we argue that such constraints are unnecessary. Our goal is not to introduce consistency as an additional heuristic property, but to reflect the fact that, in our scenario, the conditional transport path should be nearly straight. Therefore, we adopt the more aggressive choice of using arbitrary timestep pairs without additional restrictions. Second, we apply stop gradient to the prediction at the lower noise level, namely $v^{pred}_{t_2}$. The intuition is that lower noise inputs contain richer target information and thus yield more reliable predictions. We therefore treat the lower noise prediction as the target and force the higher noise prediction to align with it. Without stop gradient, both sides would be updated simultaneously, which in practice degrades performance. We provide ablations on both design choices in Appendix~\ref{appendix:ablation}.

Combining the above modifications, our conditional transport path straightening stage is trained as follows. For each training sample, we randomly choose two timesteps $t_1$ and $t_2$, and optimize the model using the original flow matching loss together with the garment preservation loss and the self consistency loss. The overall training objective is
\[
L = \sum_{t\in\{t_1,t_2\}} L_{\mathrm{fm}} + \alpha \sum_{t\in\{t_1,t_2\}} L_g + \beta L_{\mathrm{cons}} .
\]
where $\alpha$ and $\beta$ control the weights of the garment preservation loss and the self consistency loss, respectively. Unless otherwise specified, we set $\alpha=0.1$ and $\beta=0.05$. Moreover, the pretrained model is not inherently aligned with the near straight conditional transport required by VTON. We therefore adopt longer fine-tuning in this stage to reduce the mismatch between the pretrained model and the try-on task. More details are provided in Sec.~\ref{sec:ex}.

\subsection{Stage2: one-step distillation}
\label{sec:distill}
In this paper, our final goal is to obtain a one-step generator for the try-on task. Although the conditional transport path straightening stage already yields comparable one-step performance (e.g., FID 9.53 under the unpaired setting on VITON-HD), there is still a significant mismatch between training and inference, since the model in this stage is trained with 1000 timesteps but evaluated with one-step sampling. To further reduce this mismatch and improve one-step generation quality, we directly adopt Latent Adversarial Diffusion Distillation (LADD)~\citep{sauer2024fast} to distill a one-step model from the model obtained in the straightening stage. Specifically, we use the model trained in the straightening stage as the teacher model and follow the original LADD framework to train a one-step student model. Unless otherwise specified, all distillation settings follow LADD.

\section{Experiments}

\subsection{Experimental setup}
\label{sec:ex}
\paragraph{Datasets.}
In this paper, we conduct experiments on the VITON-HD~\citep{choi2021viton} and DressCode~\citep{morelli2022dress} datasets, with all ablation studies performed on VITON-HD. VITON-HD contains only upper-body garments, whereas DressCode includes three garment categories: upper-body, lower-body, and dresses. Both datasets provide paired samples, each consisting of a person image and a corresponding garment image. Such paired data cannot be directly used for mask-free training. Therefore, we first train a mask-based model to generate the required training data. Further details are provided in Appendix~\ref{appendix:ed}.

\paragraph{Implementation details.}
For network initialization, the MMDiT style backbone is initialized from Stable Diffusion 3 Medium~\citep{esser2024rectifiedflow}. Since both streams in our architecture are image streams rather than text streams, they are both initialized from the image stream of SD3 Medium. We train separate models for VITON-HD and DressCode.
In the straightening stage, we train the VITON-HD model for 400 epochs and the DressCode model for 200 epochs. For VITON-HD, the learning rate is linearly warmed up to $1\times10^{-4}$ in the first 5 epochs, kept at $1\times10^{-4}$ for 50 epochs, linearly decayed to $0$ over the next 45 epochs, and then fixed at $1\times10^{-5}$ for the remaining 300 epochs. For DressCode, we follow the same schedule pattern, with 5 warmup epochs, 25 epochs at $1\times10^{-4}$, 20 decay epochs, and 150 epochs at $1\times10^{-5}$.
In the distillation stage, we further train the VITON-HD model for 40 epochs and the DressCode model for 20 epochs, both with a learning rate of $1\times10^{-5}$. Both stages use the AdamW optimizer~\citep{loshchilov2017fixing}. Further training details, including the LADD implementation, are provided in Appendix~\ref{appendix:ed}.

\paragraph{Baselines.}
We compare our method with several recent state-of-the-art approaches, including CatVTON~\citep{chong2025catvton}, Leffa~\citep{zhou2025flowattention}, OOTDiffusion~\citep{xu2025ootdiffusion}, and 
Any2anyTryon~\citep{guo2025any2anytryon}, using their official model weights and inference code. Since these methods are not designed for sampling acceleration, we standardize the number of inference steps to 30 for all of them. We would like to compare with methods which also try to explicitly target accelerated sampling, however, CAT-DM~\citep{zeng2024cat} only released GC-DM~\citep{zeng2024cat} weights and inference code, and MC-VTON has not open-sourced the code, so we do not include them in the comparison baselines.

\paragraph{Evaluation metrics and protocols.}
Previous virtual try-on methods typically evaluate performance under both paired and unpaired settings. In the paired setting, the model reconstructs the original person image with the same garment, whereas in the unpaired setting, the garment is replaced with a different one~\citep{choi2021viton}. In this paper, since we focus on training mask-free models, we report results only under the unpaired setting, which better reflects real-world applications. Following prior works~\citep{chong2025catvton,jiang2024fitdit,zhou2025flowattention}, we adopt Fréchet Inception Distance (FID)~\citep{parmar2022aliased} and Kernel Inception Distance (KID)~\citep{binkowski2018demystifying} as quantitative metrics for the unpaired setting. We also compare inference speed on a PPU-810E, defined as the end-to-end time from input to output, including preprocessing.
All evaluations are conducted at a resolution of \(768 \times 1024\). For methods that only support \(384 \times 512\) resolution, including CatVTON and Any2anyTryon, we run inference at \(384 \times 512\) and then upsample the generated results to \(768 \times 1024\) for FID and KID evaluation. This protocol does not substantially bias the comparison. For example, evaluating CatVTON directly at \(384 \times 512\) gives an FID/KID of 9.36/1.19, while evaluating the upsampled results at \(768 \times 1024\) gives an FID/KID of 9.40/1.27. This resizing protocol is used only for FID/KID evaluation. For inference speed, we evaluate all methods directly at \(768 \times 1024\) to ensure a fair comparison.

\begin{figure}[t]
  \centering
  \includegraphics[width=\linewidth]{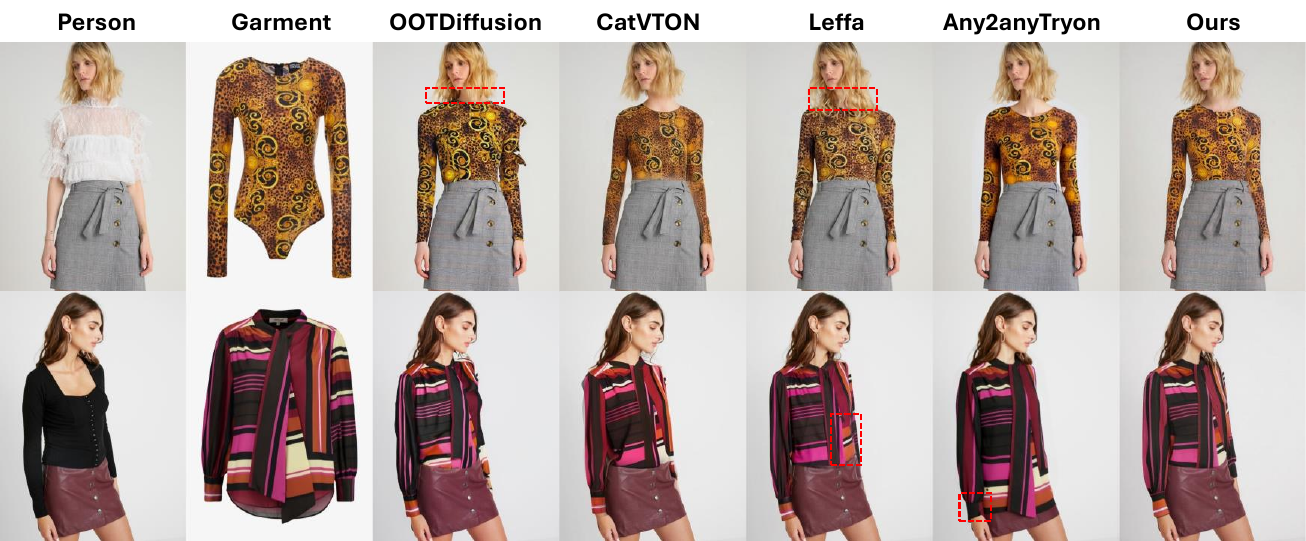}
  \caption{Qualitative comparisons on the VITON-HD~\protect\citep{choi2021viton} dataset.}
  \label{fig:viton}
\end{figure}

\begin{table*}[t]
\centering
\caption{Quantitative comparisons on the VITON-HD~\protect\citep{choi2021viton} and
DressCode~\protect\citep{morelli2022dress} datasets. Best in \textbf{bold} and second best \underline{underlined}.}
\fontsize{8pt}{9pt}
\selectfont
\begin{tabular}{cccccc}
  \toprule
  \multicolumn{1}{c}{\textbf{Dataset}} 
  & \multicolumn{2}{c}{VITON-HD} 
  & \multicolumn{2}{c}{DressCode} 
  & \multirow{2}{*}{Time (s) ↓} \\
  \cmidrule(lr){1-1}\cmidrule(lr){2-3}\cmidrule(lr){4-5}
  \multicolumn{1}{c}{\textbf{Method}}
  & FID ↓ & KID ↓
  & FID ↓ & KID ↓
  &  \\
  \midrule
  OOTDiffusion~\citep{xu2025ootdiffusion} & \underline{9.02} & \underline{0.63} & 7.10 & 2.28 & \underline{9.24} \\
  CatVTON~\citep{chong2025catvton}        & 9.40 & 1.27 & \underline{5.24}  & \underline{1.21} & 10.98 \\
  Leffa~\citep{zhou2025flowattention}     & 9.38  & 0.92 & 6.17 & 1.90 & 11.55 \\
  Any2anyTryon~\citep{guo2025any2anytryon} & 9.29  & 1.22  & - & - & 61.65 \\
  \midrule
  \textbf{DirectTryOn (ours)}                & \textbf{8.59} & \textbf{0.56} & \textbf{5.08} & \textbf{0.95} & \textbf{0.48} \\
  \bottomrule
\end{tabular}
\label{tab:viton}
\end{table*}
\subsection{Qualitative and quantitative results}
\label{subsec:compare}

\paragraph{Qualitative comparison.}
Qualitative results are shown in Fig.~\ref{fig:viton}. Although our method uses only one-step sampling, it still achieves competitive visual quality compared with multi-step baselines. In particular, one-step sampling does not degrade the try-on results, and the garment appearance is well preserved. This observation is consistent with our previous analysis: since virtual try-on is highly constrained by the person and garment conditions, the generation process is nearly deterministic, making one-step generation feasible after conditional path straightening.
We further discuss the effect of mask-based and mask-free designs. Both Any2anyTryon and our method are mask-free, while other baselines rely on explicit masks. As shown in the first row, identity-related regions, such as the hair, can sometimes be altered not only by mask-free methods but also by mask-based methods, especially when the mask is inaccurate or overly aggressive. This suggests that using an explicit mask does not necessarily guarantee better preservation of non-garment regions.

\paragraph{Quantitative comparison.}
We report quantitative results in Tab.~\ref{tab:viton}. For Any2anyTryon, since only the weights trained on VITON-HD are publicly available, we compare with it only on the VITON-HD dataset. Our method is mask-free and therefore does not require additional mask preprocessing. Moreover, since we disable CFG during inference and perform only one sampling step, our method takes only 0.48s on a PPU-810E. Even compared with OOTDiffusion, the fastest baseline among all compared methods, which takes 9.24s, our method is still about 19$\times$ faster. More importantly, despite this substantial acceleration, our method still achieves state-of-the-art performance.

\begin{figure}[t]
  \centering
  \includegraphics[width=\linewidth]{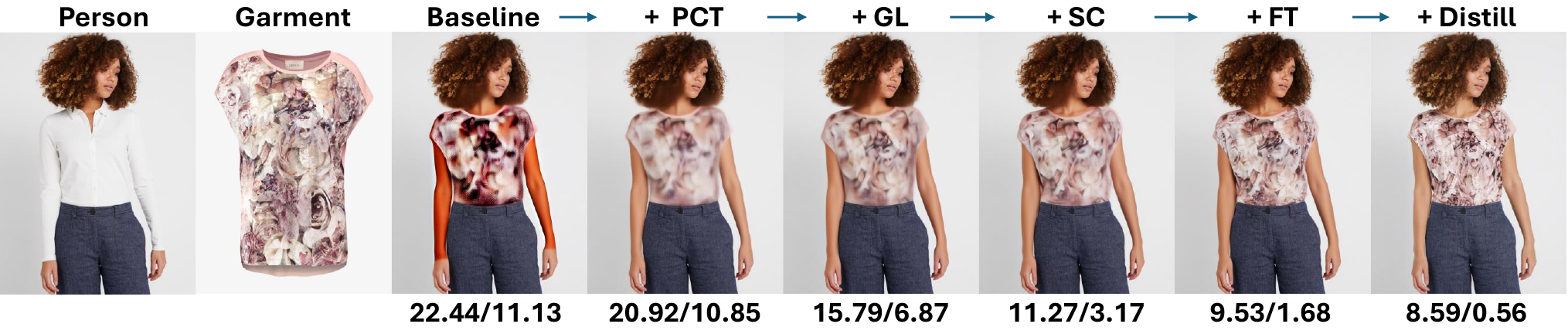}
  \caption{
   Qualitative and quantitative ablation results. From left to right, each variant progressively applies one additional modification to the previous setting: pure conditional transport (PCT), adding the garment preservation loss (GL), adding the self consistency loss (SC), using longer fine-tuning (FT), and finally applying LADD distillation. The numbers below each result denote FID/KID, where lower values indicate better performance.
   }
  \label{fig:ablation_main}
\end{figure}

\subsection{Ablation study}
\label{ablation}

To validate our proposed modifications, we conduct an ablation study on VITON-HD in the mask-free and unpaired setting. As shown in Fig.~\ref{fig:ablation_main}, the baseline uses unconditional training samples by dropping the garment condition with a ratio of 0.2, and applies CFG at inference with a guidance scale of 2. For one-step VTON, both designs are harmful: unconditional samples introduce an additional transport path that is not fully specified by the garment-person condition, while CFG extrapolates the sampling direction using the unconditional prediction, which can push the one-step result away from the target conditional distribution. Removing both yields cleaner conditional transport and substantially improves the result.
Adding the garment preservation loss and self consistency loss further helps the model learn a straighter conditional transport path. Up to this point, all variants are trained for 100 epochs. Since pretrained weights are not naturally aligned with this near straight conditional structure, extending fine-tuning by 300 epochs further improves performance, yielding the straightened teacher model. Finally, applying LADD further reduces the gap between training and one-step inference, leading to the best overall performance.

\section{Conclusion}

In this paper, we study one-step virtual try-on from the perspective of conditional transport. We argue that, compared with general image generation, virtual try-on is a more strongly conditioned task, in which the target image is largely constrained by the person image and the garment image. This suggests that its conditional transport path can be much straighter, making one-step generation a natural solution. Based on this observation, we propose a framework consisting of a conditional path straightening stage and a distillation stage. Extensive experiments on VITON-HD and DressCode show that our method achieves strong one-step try-on performance.

We hope this work can also provide useful inspiration for other generation tasks with similarly strong conditional constraints, where exploiting straighter conditional transport paths may lead to more efficient generation.

\newpage
\bibliographystyle{plainnat}
\bibliography{neurips_2026}

\appendix
\setcounter{theorem}{0}

\section{Proof of Theorem~1}
\label{appendix:proof_theorem1}

\begin{theorem}
Under the Optimal Transport conditional path in flow matching, if the conditional input uniquely determines the target, then the conditional transport path is straight, and one-step sampling is equivalent to multi-step sampling.
\end{theorem}

\begin{proof}
Under the Optimal Transport conditional path in flow matching,
\[
X_t=(1-t)X_0+tX_1,
\]
the conditional velocity field can be written as
\[
v_t(x\mid c)=\mathbb{E}\!\left[\frac{X_1-x}{1-t}\,\middle|\,X_t=x,\ c\right].
\]

As the conditional input $c$ uniquely determines the target, i.e.,
\[
X_1=y(c).
\]
Therefore,
\[
v_t(x\mid c)=\frac{y(c)-x}{1-t}.
\]

For one-step sampling, starting from an arbitrary noise $\epsilon$ at $t=0$,
\[
x_0=\epsilon.
\]
Then
\[
x_1 = x_0 + v_0(x_0\mid c)
     = \epsilon + \bigl(y(c)-\epsilon\bigr)
     = y(c).
\]

For multi-step sampling, let
\[
0=t_0<t_1<\cdots<t_N=1,
\]
and define the update
\[
x_{t_{k+1}} = x_{t_k} + (t_{k+1}-t_k)\,v_{t_k}(x_{t_k}\mid c).
\]
We prove by induction that
\[
x_{t_k}=(1-t_k)\epsilon+t_k\,y(c), \qquad \forall k.
\]
Therefore, all sampled points lie on the same linear path from $\epsilon$ to $y(c)$, which implies that the conditional transport path is straight.
 
For the base case,
\[
x_{t_0}=x_0=\epsilon=(1-t_0)\epsilon+t_0\,y(c).
\]

Assume
\[
x_{t_k}=(1-t_k)\epsilon+t_k\,y(c).
\]
Then
\[
v_{t_k}(x_{t_k}\mid c)
=\frac{y(c)-x_{t_k}}{1-t_k}
=\frac{y(c)-\bigl((1-t_k)\epsilon+t_k\,y(c)\bigr)}{1-t_k}
= y(c)-\epsilon.
\]
Therefore,
\begin{align*}
x_{t_{k+1}}
&= x_{t_k} + (t_{k+1}-t_k)\,v_{t_k}(x_{t_k}\mid c) \\
&= (1-t_k)\epsilon+t_k\,y(c) + (t_{k+1}-t_k)\bigl(y(c)-\epsilon\bigr) \\
&= (1-t_{k+1})\epsilon+t_{k+1}\,y(c).
\end{align*}
Thus the induction holds.

Finally, at $t_N=1$,
\[
x_{t_N}=(1-t_N)\epsilon+t_N\,y(c)=y(c).
\]
Hence multi-step sampling reaches the same final result as one-step sampling.
\end{proof}

\section{More experiment details}
\label{appendix:ed}
\paragraph{Training data.}
In this paper, we mainly focus on training a mask-free model for one-step generation. Compared with mask-based training, which only requires paired data consisting of a person image and the corresponding garment image, mask-free training requires at least three images: two person images of the same identity wearing different garments, and one garment image corresponding to one of these two person images. Therefore, when using existing datasets such as VITON-HD~\citep{choi2021viton} and DressCode~\citep{morelli2022dress}, which only provide paired data, we need to generate an additional person image wearing a different garment for each pair in order to construct mask-free training data.
To this end, we first train a mask-based model and use it to synthesize the required training data. This mask-based model adopts the same network architecture as our mask-free model, with the main difference lying in the inputs. Specifically, the person image is replaced by an agnostic image, where the garment region is masked out using segmentation information~\citep{gong2018instance,cao2019openpose}. In addition, after the patch embedding layer, we add the DensePose image~\citep{guler2018densepose} of the person to the noisy latent, so that the generated result is encouraged to preserve the original pose. This design helps avoid pose inconsistency caused by inaccurate masks.

\paragraph{LADD implementation.}
In the distillation stage, we adopt Latent Adversarial Diffusion Distillation (LADD)~\citep{sauer2024fast} to compress our straightened teacher model into a one-step student, with one-step generation as the final target. We first use the straightened teacher model offline to generate pseudo ground-truth try-on images. Specifically, for each training sample, we randomly pair a person image with a garment image from the dataset, and use the teacher model to synthesize the corresponding try-on result. The generated result is then treated as the target image for student training.
During distillation, the student is initialized from the teacher model and takes the pure noisy target latent concatenated with the person latent as input, while using the garment latent as the try-on condition. It predicts the velocity field, from which we recover the clean latent prediction. The student is trained with a reconstruction loss between its predicted clean latent and the teacher-generated target latent.
In addition to this reconstruction objective, we follow LADD and introduce a latent adversarial objective based on frozen teacher features. Instead of using an external image-space discriminator, we re-noise both the teacher-generated target latent and the student-predicted latent to randomly sampled noise levels, and feed them into the frozen teacher transformer. We extract intermediate hidden states from multiple transformer blocks and attach lightweight convolutional discriminator heads to these features. The discriminator is trained to distinguish re-noised teacher targets from re-noised student predictions, while the student is trained to fool the discriminator. Since our model is initialized from Stable Diffusion 3-Medium, which contains 24 transformer blocks, we sample 8 blocks across the transformer depth to provide multi-level adversarial supervision.

\paragraph{Training details.}
All experiments in this paper are conducted on a cluster of 16 PPU-810E processors, each with 100 GB of memory. For the path-straightening stage, the batch size is set to 8 per processor, resulting in a total batch size of 128. Training takes approximately 6 days for 400 epochs on VITON-HD and 11 days for 200 epochs on DressCode.
For the distillation stage, the batch size is set to 3 per processor, resulting in a total batch size of 48. Distillation takes approximately 9 hours for 40 epochs on VITON-HD and 18 hours for 20 epochs on DressCode.

\begin{table*}[t]
  \centering
  \caption{Ablation study on two design choices in the self consistency loss.}
  \label{tab:ablation_sc}
  \begin{tabular}{ccc}
    \toprule
    Version & FID~\(\downarrow\) & KID~\(\downarrow\) \\
    \midrule
    Fixed Timestep Interval & 14.15 & 5.52 \\
    Without StopGrad & 13.26 & 4.78 \\
    \textbf{Ours} & \textbf{11.27} & \textbf{3.17} \\
    \bottomrule
  \end{tabular}
\end{table*}

\section{More ablation studies}
\label{appendix:ablation}

Here we provide additional ablations on the design of our self consistency loss. Our method contains two special design choices: first, we do not impose any constraint on timestep sampling for the consistency loss; second, we apply \texttt{stopgrad} to the lower-noise prediction, rather than simply requiring the predictions at two timesteps to be the same.
For this ablation study, we use mask-free training and evaluate under the unpaired setting on VITON-HD. All models are trained for 100 epochs. In all three variants, pure conditional transport and the garment preservation loss are applied. The only difference lies in the design of the self consistency loss. The original version used in our method is denoted as \textbf{Ours}. The variant \textbf{Fixed Timestep Interval} requires the two sampled timesteps to have a fixed interval of 100, for example 900 and 800. The variant \textbf{Without StopGrad} removes the \texttt{stopgrad} operator from the original design. As shown in Tab.~\ref{tab:ablation_sc}, our original design achieves the best performance.

\begin{figure}[t]
  \centering
  \includegraphics[width=\linewidth]{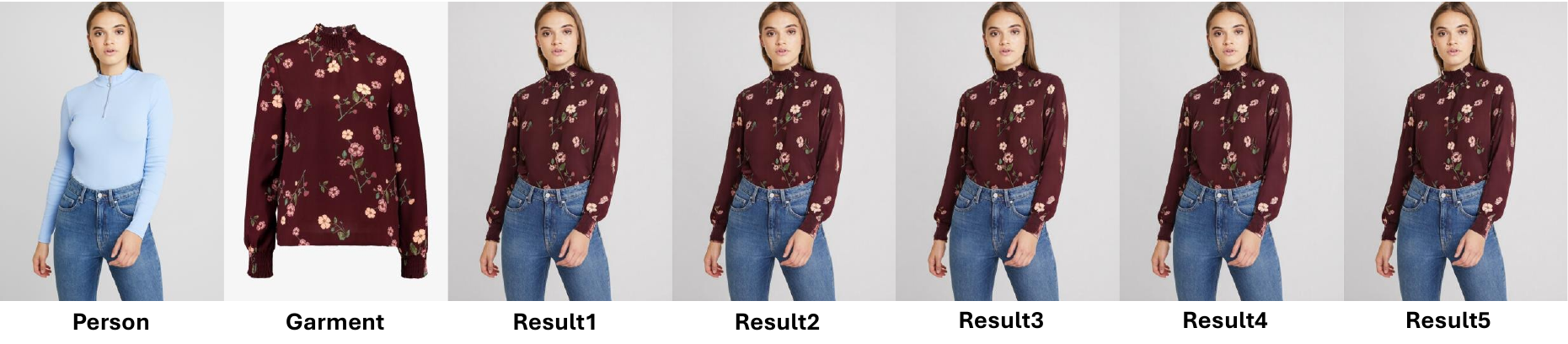}
  \caption{Results of our distilled one-step student model using five different Gaussian noise initializations. For each initialization, the model performs a single denoising step from the sampled Gaussian noise to the final try-on result. Despite starting from different noise samples, the generated results remain highly similar, with only slight variations in garment wrinkles, demonstrating the robustness of one-step inference to the choice of initial noise.}
  \label{fig:mul_s}
\end{figure}

\section{Robustness to noise initializaion}
\label{appendix:test}

Here, we provide an additional test of the distilled one-step student model. We randomly sample five different Gaussian noises, while keeping the same garment image and person image as input, and perform one-step inference from each noise. As shown in Fig.~\ref{fig:mul_s}, the generated results are highly similar, differing only slightly in garment wrinkles. This confirms that our one-step model is robust to different Gaussian noise initializations and does not suffer noticeable quality degradation from different starting noise points.

\section{More analysis}
Recall the CFG ablation in Sec.~\ref{sec:s_stage}. Fig.~\ref{fig:cfg_ablation} reveals several additional observations. First, $\mathbf{C}_{\text{UT,CFG}}^{30}$ consistently achieves the best or near-best performance. This is expected: when CFG is enabled at inference and multi-step sampling is used, introducing unconditional training samples helps the model learn a more reliable velocity field and improves the effectiveness of guidance.
Second, comparing $\mathbf{C}_{\text{noUT,noCFG}}^{30}$ with $\mathbf{C}_{\text{UT,noCFG}}^{30}$, we observe that unconditional training still improves performance even when CFG is disabled during inference. This suggests that unconditional samples are not only useful for CFG, but also regularize the model during training. In the 30-step setting, this regularization allows the model to better explore the underlying data distribution and produce more robust predictions.

\section{Limitations}
In this paper, we do not focus on the distinction between mask-based and mask-free settings. Instead, we train our own mask-based model to generate training data for the mask-free setting. Although this strategy provides relatively reliable training data, the generated data is still not perfect. As a result, our method may inherit a common issue in previous methods~\citep{sun2025ds}: regions unrelated to the target garment, such as hair or accessories, may occasionally be altered. Fortunately, this issue occurs infrequently in practice, and even when it occurs, the artifacts are usually mild.

\section{Broader impacts}
The ability of DirectTryOn to generate realistic virtual try-on results efficiently makes it well suited for practical deployment in e-commerce scenarios. At the same time, as with other generative technologies, DirectTryOn may raise concerns related to intellectual property and personal privacy. We encourage its responsible and ethical use.

\section{Additional qualitative results on DressCode}
In this section, we present additional qualitative comparison results on the DressCode~\citep{morelli2022dress} dataset.

\begin{figure}[p]
  \centering
  \includegraphics[width=\linewidth]{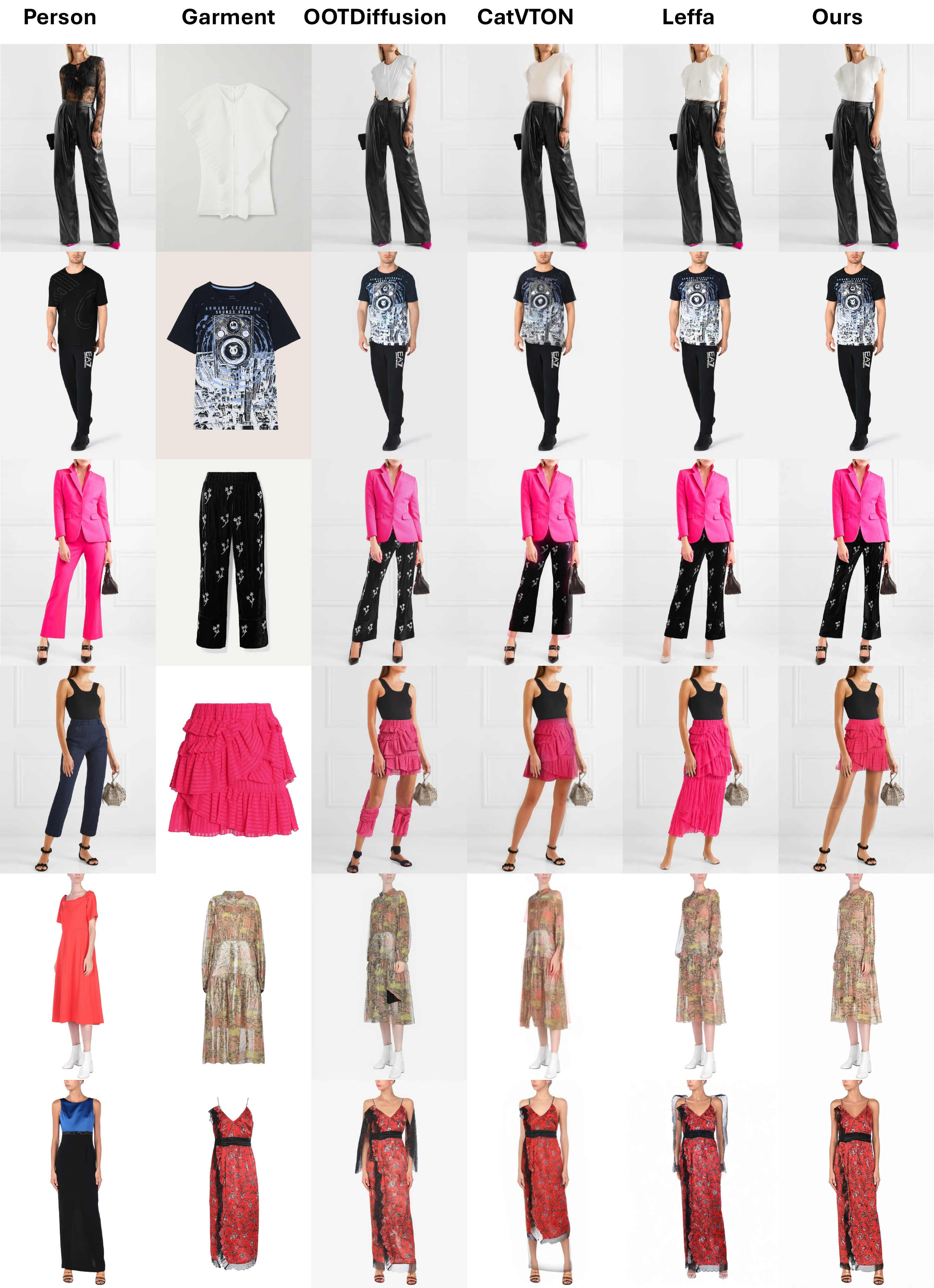}
  \caption{Qualitative comparison on the DressCode dataset.}
  \label{fig:Compare_DressCode_Dress}
\end{figure}


\newpage

\end{document}